\documentclass{article}

\usepackage{microtype}
\usepackage{graphicx}
\usepackage{subcaption}
\usepackage{booktabs}
\usepackage{amsmath}
\usepackage{amssymb}
\usepackage{mathtools}
\usepackage{amsthm}
\usepackage{enumitem}
\usepackage[nobiblatex]{xurl}
\usepackage{xcolor}   
\usepackage{pifont}

\usepackage{hyperref}
\usepackage[capitalize,noabbrev]{cleveref}

\usepackage[preprint]{icml2026}

\makeatletter
\@ifundefined{algorithm}{}{}
\@ifundefined{endalgorithm}{}{}
\@ifundefined{listofalgorithms}{}{}

\expandafter\ifx\csname algorithm*\endcsname\relax\else
  \expandafter\let\csname algorithm*\endcsname\relax
\fi
\expandafter\ifx\csname endalgorithm*\endcsname\relax\else
  \expandafter\let\csname endalgorithm*\endcsname\relax
\fi
\makeatother

\usepackage[ruled,vlined,linesnumbered]{algorithm2e}

\theoremstyle{plain}
\newtheorem{theorem}{Theorem}[section]

\theoremstyle{definition}
\newtheorem{definition}[theorem]{Definition}

\theoremstyle{remark}

\usepackage[textsize=tiny]{todonotes}

\icmltitlerunning{Differentially Private Multimodal In-Context Learning}

\begin{document}
\twocolumn[
  \icmltitle{ Differentially Private Multimodal In-Context Learning}



  \icmlsetsymbol{equal}{*}
\begin{icmlauthorlist}
  \icmlauthor{Ivoline C. Ngong}{uvm}
  \icmlauthor{Zarreen Reza}{ind}
  \icmlauthor{Joseph P. Near}{uvm}
\end{icmlauthorlist}

\icmlaffiliation{uvm}{University of Vermont, Burlington, VT, USA}
\icmlaffiliation{ind}{Independent Researcher}

\icmlcorrespondingauthor{Ivoline C. Ngong}{ivoline.ngong@uvm.edu}



  \icmlkeywords{Machine Learning, ICML}

  \vskip 0.3in
]



\printAffiliationsAndNotice{}  

\begin{abstract}
Vision-language models are increasingly applied to sensitive domains such as medical imaging and personal photographs, yet existing differentially private methods for in-context learning are limited to few-shot, text-only settings because privacy cost scales with the number of tokens processed. We present Differentially Private Multimodal Task Vectors (DP-MTV), the first framework enabling many-shot multimodal in-context learning with formal $(\varepsilon, \delta)$-differential privacy by aggregating hundreds of demonstrations into compact task vectors in activation space. DP-MTV partitions private data into disjoint chunks, applies per-layer clipping to bound sensitivity, and adds calibrated noise to the aggregate, requiring only a single noise addition that enables unlimited inference queries. We evaluate on eight benchmarks across three VLM architectures, supporting deployment with or without auxiliary data. At $\varepsilon=1.0$, DP-MTV achieves 50\% on VizWiz compared to 55\% non-private and 35\% zero-shot, preserving most of the gain from in-context learning under meaningful privacy constraints. 
\end{abstract}

\section{Introduction}

In-context learning (ICL) enables vision-language models to adapt to new tasks by conditioning on demonstration examples at inference time, without finetuning~\citep{brown2020language, dong2022survey}. This allows organizations to customize models using their own data directly. However, when that data contains private information, ICL poses serious 
privacy risks. Models can memorize and leak content from demonstrations 
through attacks such as membership inference, data extraction, and 
prompt leaking~\cite{wen2024membership, duan2024privacy, hou2025neighborhood}

Consider a tax preparation service using a VLM to process client documents. Each image contains Social Security numbers, addresses, and income data that VLMs can extract through OCR~\citep{liu2024ocrbench}. An adversary with black-box access could query for this information directly or perform membership inference to determine whether a specific client's data was used. Images also leak information through visual content as models can infer geolocation from scene details~\citep{mendes2024granular}, extract private attributes from environmental cues ~\citep{tomekcce2024private}, and
reveal biometric characteristics despite anonymization ~\citep{caldarella2024phantom, kim2025safe}.



Differential privacy (DP) provides formal guarantees that limit what an adversary can infer about any individual in the
dataset~\citep{dwork2006calibrating}. Recent work has applied DP to text-based in-context learning by generating private synthetic demonstrations~\citep{panda2023differentially,
tang2023privacy}, aggregating outputs across disjoint exemplar sets ~\citep{wu2023privacy, duan2024privacy}, applying local differential privacy to labels ~\citep{zheng2024locally}, leveraging auxiliary public data ~\citep{joo2025public}, and mixing few-shot with zero-shot outputs ~\citep{flemings2025differentially}. Despite their differences, all these methods work only with text and are restricted to few-shot settings. This restriction arises because privacy cost accumulates with either the number of demonstrations or the tokens generated, and the noise required for strong privacy at scale destroys utility. For multimodal data, where a single image consumes hundreds of tokens~\citep{bai2023qwen}, even a handful of demonstrations exhaust both context capacity and privacy budget. To our knowledge, no prior work addresses differentially private in-context learning for multimodal data.

One way to scale beyond context limits is Multimodal Task 
Vectors (MTV)~\citep{huang2024multimodal}, which aggregates activation patterns from hundreds of examples into compact steering vectors applied at inference time. This bypasses context limits entirely, but  MTV provides no privacy guarantees. The vectors directly encode patterns from demonstrations, exposing them to potential inference attacks. Practitioners must therefore choose between privacy risks with many-shot learning or few-shot protection that does not extend to images.

We introduce Differentially Private Multimodal Task Vectors (DP-MTV), the first framework enabling many-shot multimodal in-context learning with formal $(\varepsilon, \delta)$-differential privacy. Operating in activation space fundamentally changes the privacy problem. Instead of protecting each token or demonstration individually, we aggregate activation patterns first and privatize the aggregate. DP-MTV partitions private data into disjoint chunks where each individual's data appears exactly once, clips each chunk's contribution to bound sensitivity, and adds calibrated Gaussian noise to the mean. This requires only a single noise addition regardless of dataset size, enabling unlimited inference queries at zero additional privacy cost.

We evaluate DP-MTV on eight benchmarks spanning visual question answering and fine-grained classification using three VLM architectures. At $\varepsilon=1.0$, DP-MTV achieves 49\% accuracy on VizWiz compared to 55\% for non-private MTV and 35\% for zero-shot, preserving most of the gain from in-context learning under meaningful privacy constraints. We provide both a variant that uses public auxiliary data for efficiency and a fully-private variant requiring no auxiliary data.

\paragraph{Contributions.}
\begin{itemize}[leftmargin=*, itemsep=2pt, topsep=2pt]
\item We introduce DP-MTV, the first method for differentially
private many-shot multimodal in-context learning, enabling formal $(\varepsilon, \delta)$-DP guarantees for learning from hundreds of image-text demonstrations.

\item We show that operating in activation space with disjoint partitioning and per-layer clipping requires only a single noise addition, enabling unlimited inference queries at zero marginal
privacy cost.

\item We evaluate on eight benchmarks across three VLM architectures, demonstrating that formal privacy guarantees are achievable without sacrificing the core benefit of learning from many examples.
\end{itemize}


\section{Preliminaries}

\subsection{Differential Privacy}

Differential privacy(DP)~\citep{dwork2006calibrating, dwork2014algorithmic} provides formal guarantees that an algorithm's output reveals limited information about any individual input record.

\begin{definition}[$(\varepsilon, \delta)$-Differential Privacy]
\label{def:dp}
A randomized mechanism $\mathcal{M}$ satisfies $(\varepsilon, \delta)$-differential privacy if for any datasets $\mathcal{D}, \mathcal{D}'$ differing in a single record, and any measurable set $\mathcal{O}$:
\begin{equation}
    \Pr[\mathcal{M}(\mathcal{D}) \in \mathcal{O}] \leq e^{\varepsilon} \cdot \Pr[\mathcal{M}(\mathcal{D}') \in \mathcal{O}] + \delta
\end{equation}
\end{definition}

The privacy budget $\varepsilon$ controls the strength of the guarantee: smaller values provide stronger protection. The parameter $\delta$ represents the probability of a privacy breach and is typically set to be cryptographically small (e.g., $\delta < 1/|\mathcal{D}|$).

We rely on three standard results. The \emph{Gaussian mechanism}~\citep{dwork2014algorithmic} achieves $(\varepsilon, \delta)$-DP by adding noise calibrated to a query's sensitivity: for $f: \mathcal{D} \rightarrow \mathbb{R}^d$ with $\ell_2$-sensitivity $\Delta_2 = \max_{\mathcal{D}, \mathcal{D}'} \|f(\mathcal{D}) - f(\mathcal{D}')\|_2$, releasing $f(\mathcal{D}) + \mathcal{N}(0, \sigma^2 I)$ with $\sigma \geq \Delta_2 \sqrt{2 \ln(1.25/\delta)} / \varepsilon$ satisfies $(\varepsilon, \delta)$-DP. Our implementation uses the analytic Gaussian mechanism~\cite{balle2018improving} which provides tighter bounds on privacy loss. The \emph{composition theorem} states that releasing the outputs of an $(\varepsilon_1, \delta_1)$-DP mechanism and an $(\varepsilon_2, \delta_2)$-DP mechanism together satisfies $(\varepsilon_1 + \varepsilon_2, \delta_1 + \delta_2)$-DP. The \emph{post-processing property} guarantees that any computation on the output of an $(\varepsilon, \delta)$-DP mechanism remains $(\varepsilon, \delta)$-DP.

\subsection{Multimodal Task Vectors}

Multimodal Task Vectors (MTV)~\citep{huang2024multimodal} enable many-shot in-context learning for vision-language models (VLMs) by encoding task knowledge into attention head activations rather than token sequences. Consider a VLM with $L$ transformer layers, each containing $H$ attention heads of dimension $d$. MTV operates in three steps.

First, MTV processes multiple batches of in-context demonstrations through the VLM. For each forward pass, it extracts the attention head activations at the final token position, yielding a tensor in $\mathbb{R}^{L \times H \times d}$. These activations are averaged across all forward passes to produce the mean activation tensor $\bar{a} \in \mathbb{R}^{L \times H \times d}$. Second, MTV uses REINFORCE~\citep{williams1992simple} to learn a binary mask $\mathbf{m} \in \{0,1\}^{L \times H}$ identifying which attention heads benefit from activation injection. Third, at inference time, for each selected head where $\mathbf{m}_{l,h} = 1$, the model's original activation is replaced with the corresponding component of $\bar{a}$, steering behavior toward the demonstrated task.

MTV's key property is that it aggregates information from hundreds of examples into a compact representation. This enables many-shot learning that would otherwise exceed context window limits, achieving substantial accuracy gains on visual question answering and image classification benchmarks.

\subsection{Privacy in In-Context Learning}
In-context learning poses inherent privacy risks. When demonstrations contain sensitive data, models exhibit higher confidence on those examples, enabling membership inference attacks that determine whether specific records were used~\citep{wen2024membership, duan2024privacy}. 
These risks extend to multimodal settings, where images may reveal sensitive attributes beyond their apparent content~\cite{mendes2024granular, tomekcce2024private, caldarella2024phantom}

Several frameworks have introduced differential privacy to ICL. \citet{tang2023privacy} generate DP synthetic demonstrations via noisy token-level aggregation. \citet{wu2023privacypreserving} and \citet{duan2024privacy} aggregate predictions across disjoint exemplar sets using PATE-style voting. \citet{zheng2024locally} apply local DP 
to labels, while \citet{joo2025public} leverage public auxiliary data to reduce privacy costs.

These methods share two limitations. First, they operate in token space, where privacy cost scales with sequence length. A single image  corresponds to hundreds of visual tokens~\citep{bai2023qwenvl}, making token-level protection prohibitively expensive for multimodal ICL.  Second, all existing methods are restricted to few-shot settings due to context window constraints. Our work addresses both limitations by privatizing in activation space, where the cost depends on the number of DP mechanisms applied rather than tokens processed.

\label{sec:preliminaries}

\section{Differentially Private Multimodal Task Vectors}


\begin{figure*}[t]
    \centering
    \includegraphics[width=0.75\textwidth]{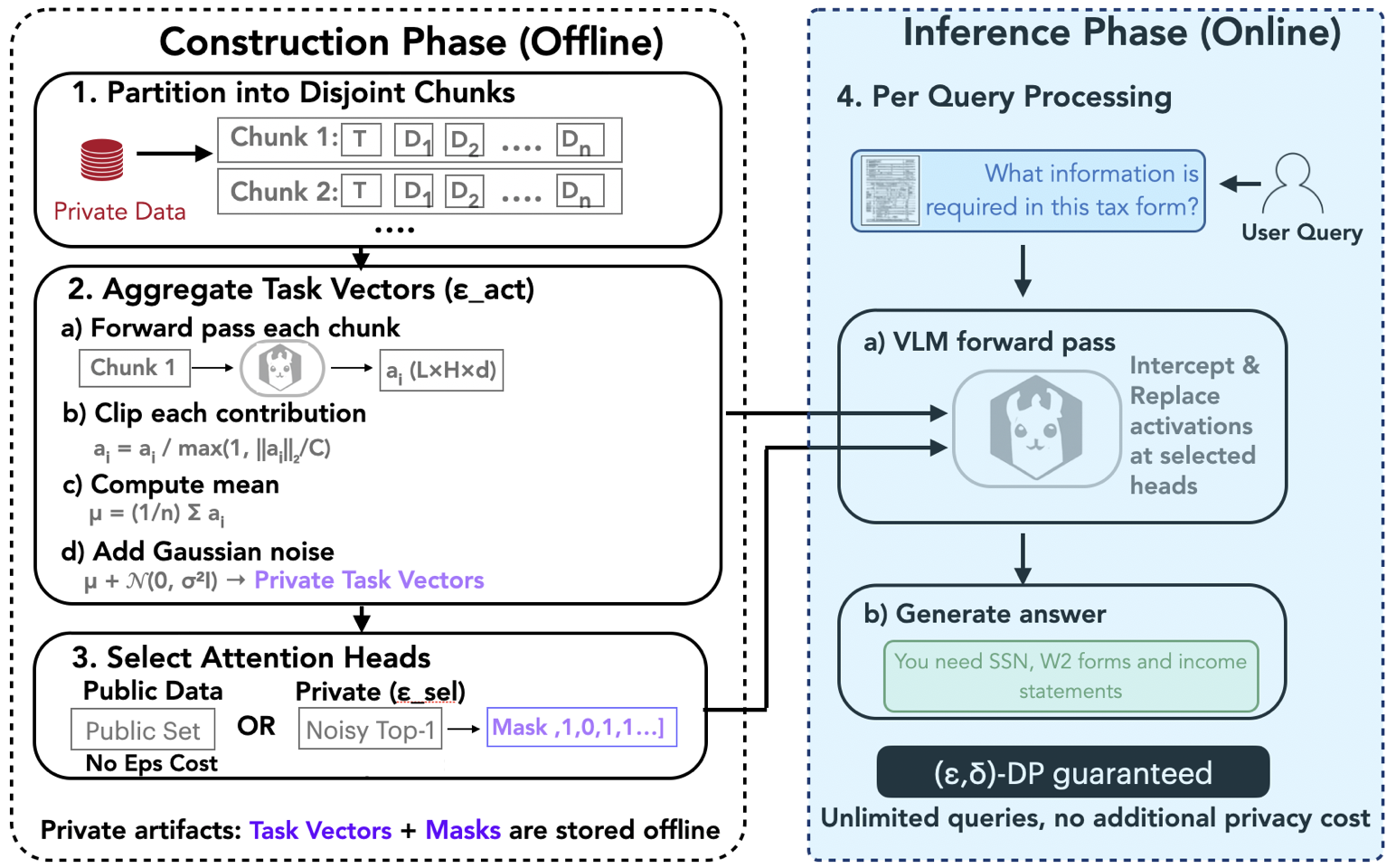}
    \caption{DP-MTV approach. \textbf{Construction (offline):} Partition data into disjoint chunks, extract and clip activations, compute the mean, add Gaussian noise (steps 1--2), then select heads via public data or a private mechanism (step 3). \textbf{Inference (online):} Replace activations at selected heads with private task vectors. Post-processing enables unlimited queries at no additional privacy cost.}
    \label{fig:dp_mtv_approach}
\end{figure*}

\subsection{Problem Setup}
\label{sec:problem_setup}

We consider a vision-language model with $L$ transformer layers, each containing $H$ attention heads of dimension $d$. Given a private dataset $\mathcal{D}_{\text{priv}}$ of $N$ image-question-answer triplets $(x, q, a)$, the goal is to enable many-shot in-context learning from $\mathcal{D}_{\text{priv}}$ while satisfying $(\varepsilon, \delta)$-differential privacy. Standard DP composition pays privacy cost for each training example accessed and each inference query served. With hundreds of examples and unlimited queries, this accumulation exhausts any finite privacy budget.

\paragraph{Threat Model.} We protect each example in $\mathcal{D}_{\text{priv}}$ against adversaries who observe model outputs and attempt to infer whether specific examples were used in training (membership inference) or extract sensitive attributes from them. The adversary can query the model arbitrarily but has no access to the training procedure or model internals.

\subsection{Overview}

\label{sec:overview}

Figure~\ref{fig:dp_mtv_approach} illustrates our approach. MTV separates construction from inference, suggesting a natural privacy strategy: privatize the mean activations and head masks once during construction, then reuse them for unlimited queries without additional privacy cost. Two technical challenges emerge: computing private mean activations without paying cost for each training example, and selecting heads privately when no public data exists.

Standard approaches pay privacy cost for each example accessed, and with hundreds of training examples, even modest budgets like $\varepsilon=3.0$ would add so much noise that the model cannot learn from the data. Head selection poses a second challenge, since evaluating which heads are task-relevant requires labeled examples that also come from private data when no public alternative exists.

Our method addresses both challenges through a two-phase approach (Figure~\ref{fig:dp_mtv_approach}). The \textbf{construction phase (offline)} partitions data into disjoint chunks so each example appears once, clips each chunk's activations to bounded norm, computes the mean, and adds calibrated noise to protect privacy. Head selection either runs on public data at zero privacy cost or uses a private selection mechanism when no public data exists. The \textbf{inference phase (online)} processes queries by replacing activations at selected heads with the private mean, requiring no additional privacy cost.

\subsection{Construction Phase}
\subsubsection{Private Mean Activations}
Computing private mean activations requires balancing two competing goals: tight sensitivity bounds that minimize noise, and preserving enough signal for effective learning. We achieve this balance through disjoint partitioning combined with per-layer clipping.

We partition $\mathcal{D}_{\text{priv}}$ into $m$ disjoint chunks where each example appears exactly once. Each chunk contains one target example plus $K$ demonstration examples. For VQA tasks, demonstrations are image-question-answer triplets that provide context. For classification tasks, demonstrations are image-label pairs showing examples of the relevant classes. Each chunk undergoes a forward pass through the VLM to extract activation patterns at selected layers $\mathcal{S} \subseteq [L]$, producing activation tensor $a_i \in \mathbb{R}^{|\mathcal{S}| \times H \times d}$.

Before aggregation, we clip activations independently for each layer. For layer $l \in \mathcal{S}$, the layer-level activation $a_i^{(l)} \in \mathbb{R}^{H \times d}$ is clipped to norm at most $C$:
\begin{equation}
\tilde{a}_i^{(l)} = a_i^{(l)} / \max(1, \|a_i^{(l)}\|_2 / C)
\end{equation}

After clipping, we compute the mean 
$\bar{a} = \frac{1}{m} \sum_{i=1}^m \tilde{a}_i$.
Since each example appears in exactly one chunk and clipping bounds each chunk’s contribution, changing one example affects only one clipped tensor. 
The $\ell_2$-sensitivity is
\begin{equation}
\Delta_2 = \frac{\sqrt{|\mathcal{S}|} \cdot C}{m}.
\end{equation}
The $\sqrt{|\mathcal{S}|}$ factor arises because clipping is applied independently per layer, and these contributions compose in $\ell_2$ norm. 
With $|\mathcal{S}| = 32$, this yields $\Delta_2 \approx \frac{5.66C}{m}$.

We add Gaussian noise calibrated to this sensitivity using the analytic Gaussian mechanism~\cite{balle2018improving}.
We release:
\begin{equation}
\bar{a}^{\text{priv}} = \bar{a} + \mathcal{N}(0, \sigma^2 I).
\end{equation}
This satisfies $(\varepsilon_{\text{tv}}, \delta)$-differential privacy
(Theorem~\ref{thm:task_vector_privacy}).

\begin{algorithm}[t]
\caption{Private Mean Activation Computation}
\label{alg:private_activations}
\SetKwComment{Comment}{$\rhd$\ }{}

\KwIn{$\mathcal{D}_{\text{priv}}$ (private dataset), $F$ (VLM), $\mathcal{S}$ (selected layers), $K$ (demonstrations per chunk), $C$ (clipping threshold), $\varepsilon_{\text{tv}}, \delta$ (privacy parameters)}
\KwOut{$\bar{a}^{\text{priv}} \in \mathbb{R}^{|\mathcal{S}| \times H \times d}$ (private mean activations)}

Partition $\mathcal{D}_{\text{priv}}$ into $m$ disjoint chunks $\mathcal{C}_i = \{\text{target}_i, \text{demo}_{i,1}, \ldots, \text{demo}_{i,K}\}$ \Comment*[r]{Each example in exactly one chunk}

\For{$i \gets 1$ \KwTo $m$}{
    $a_i \gets \text{Extract}(F, \mathcal{C}_i, \mathcal{S})$ \Comment*[r]{Forward pass chunk $\mathcal{C}_i$}
    \For{$l \in \mathcal{S}$}{
        $\tilde{a}_i^{(l)} \gets a_i^{(l)} \cdot \min\left(1, \frac{C}{\|a_i^{(l)}\|_2}\right)$ \Comment*[r]{Per-layer clipping}
    }
}

$\bar{a} \gets \frac{1}{m} \sum_{i=1}^m \tilde{a}_i$ \Comment*[r]{Aggregate clipped activations}

$\Delta_2 \gets \frac{\sqrt{|\mathcal{S}|} \cdot C}{m}$, \quad $\sigma^2 \gets \text{AnalyticGaussian}(\Delta_2, \varepsilon_{\text{tv}}, \delta)$ \Comment*[r]{Sensitivity and noise scale}

$\bar{a}^{\text{priv}} \gets \bar{a} + \mathcal{N}(0, \sigma^2 I)$ \Comment*[r]{Add Gaussian noise}

\Return{$\bar{a}^{\text{priv}}$}
\end{algorithm}


\subsubsection{Attention Head Selection}
\label{sec:head_selection}

Mean activations must be injected into task-relevant attention heads. We use REINFORCE~\citep{williams1992simple} to learn which heads to modify by training Bernoulli distributions over attention heads and sampling binary masks that maximize task performance via policy gradients. We consider two scenarios based on data availability.

\paragraph{Public-Data Variant.} When public examples from a related distribution exist, we run REINFORCE entirely on this public data at zero privacy cost. The public data need not match $\mathcal{D}_{\text{priv}}$ exactly—similar task structure suffices. For medical VQA, public examples might come from different hospitals or imaging modalities. For document understanding, public forms from different domains provide adequate signal. This variant concentrates the full privacy budget $(\varepsilon_{\text{tv}}, \delta)$ on privatizing mean activations. The final system satisfies $(\varepsilon_{\text{tv}}, \delta)$-DP since head selection with post-processing preserves privacy (Theorem~\ref{thm:public_variant}).

\paragraph{Private-Only Variant.} When no public data exists, we privatize the selection of the final mask via noisy top-$k$ selection with limited domain~\citep{durfee2019practical}. We sample candidate masks from the learned Bernoulli distribution and retain the most frequently occurring ones. For each candidate, we compute an aggregate score by evaluating it on private validation examples, clipping per-example cross-entropy losses to a bounded range and summing them. We then apply the Gumbel mechanism~\citep{durfee2019practical}, adding Gumbel noise with scale $C_{\text{sel}}/\varepsilon_{\text{sel}}$ to each score and selecting the mask with the highest noisy score. Since we select from a limited domain of candidate masks rather than all $2^{|\mathcal{S}| \times H}$ possible masks, tight privacy bounds are achievable even with small $\varepsilon_{\text{sel}}$ values. The total privacy cost is $(\varepsilon_{\text{tv}} + \varepsilon_{\text{sel}}, \delta )$ by composition (Theorem~\ref{thm:private_variant}). Slightly tighter bounds could be obtained using a numerical approach~\cite{gopi2021numerical}, but the practical benefit would be small since we compose only two mechanisms.

\subsection{Inference Phase}
During inference, the model processes queries using only the private mean activations $\bar{a}^{\text{priv}}$ and head mask $\mathbf{m}$ computed during construction. For each query, the model performs a standard forward pass. At each selected attention head where $\mathbf{m}_{l,h} = 1$, we intercept the computed activation and replace it with the corresponding component of $\bar{a}^{\text{priv}}$ before continuing the forward pass. The model then generates an answer using these modified activations.

This inference procedure satisfies the post-processing property of differential privacy. Since the private artifacts $(\bar{a}^{\text{priv}}, \mathbf{m})$ were released with privacy guarantees during construction, any deterministic computation on these artifacts preserves those guarantees. The model can serve unlimited queries without accumulating additional privacy cost, making deployment practical for real-world applications where query volume is unpredictable.





\subsection{Privacy Analysis}
\label{sec:privacy_analysis}

We establish privacy guarantees for both deployment variants. Full proofs appear in Appendix~\ref{app:proofs}.

\begin{theorem}[Private Mean Activations]
\label{thm:task_vector_privacy}

Disjoint partitioning ensures each record affects exactly one chunk, and per-layer clipping bounds each chunk's contribution. The $\ell_2$-sensitivity is $\Delta_2 = \sqrt{|\mathcal{S}|} \cdot C / m$, which is $\sqrt{H}$ times smaller than per-head clipping, reducing noise by a factor of $\sqrt{H} \approx 5.66\times$ for models with $H=32$ heads.
\end{theorem}

\begin{theorem}[Public-Data Variant]
\label{thm:public_variant}
When head selection runs entirely on public data $\mathcal{D}_{\text{pub}}$, the complete system releasing $(\bar{a}^{\text{priv}}, \mathbf{m})$ satisfies $(\varepsilon_{\text{tv}}, \delta)$-differential privacy with respect to $\mathcal{D}_{\text{priv}}$.
\end{theorem}

Head selection on public data incurs zero privacy cost. Inference applies a deterministic function to the released artifacts; by the post-processing property, this preserves the $(\varepsilon_{\text{tv}}, \delta)$ guarantee for unlimited queries.

\begin{theorem}[Private-Only Variant]
\label{thm:private_variant}
When head selection uses the Gumbel mechanism over $\bar{k}$ candidates with loss clipping threshold $C_{\text{sel}}$, the complete system satisfies $(\varepsilon_{\text{tv}} + \varepsilon_{\text{sel}}, \delta)$-differential privacy.
\end{theorem}

Mean activations satisfy $(\varepsilon_{\text{tv}}, \delta)$-DP by Theorem~\ref{thm:task_vector_privacy}. The Gumbel mechanism provides $(\varepsilon_{\text{sel}}, 0)$-DP for selecting from the limited domain~\citep{durfee2019practical}. Basic composition yields the total guarantee. Inference remains post-processing.

\section{Experiments \& Results}


\subsection{Experimental Setup}
\label{sec:setup}

\begin{figure*}[t]
\centering
\begin{subfigure}[b]{\textwidth}
    \centering
    \includegraphics[width=0.95\textwidth]{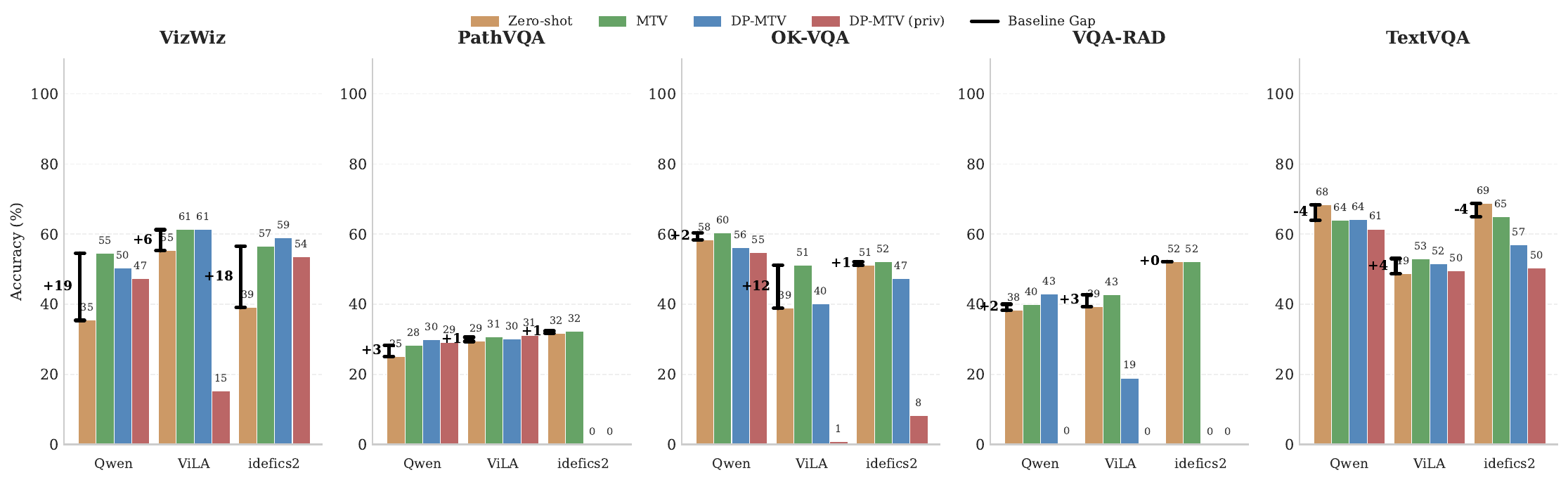}
    \caption{VQA Tasks}
    \label{fig:fig_vqa_model_comparison}
\end{subfigure}

\vspace{0.3cm}

\begin{subfigure}[b]{\textwidth}
    \centering
    \includegraphics[width=0.5\textwidth]{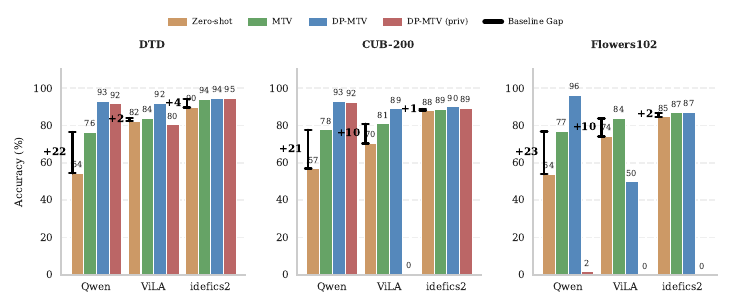}
    \caption{Classification Tasks}
    \label{fig:fig_classification_model_comparison}
\end{subfigure}
\caption{Model comparison at $\varepsilon=1.0$. Brackets show the \emph{baseline gap} (MTV $-$ Zero-shot); larger gaps predict better DP-MTV performance. (a)~VizWiz has the largest gaps and strongest DP-MTV results. (b)~On 2-way classification, DP-MTV often matches or exceeds MTV.}
\label{fig:model_comparison}
\end{figure*}


\begin{figure*}[t]
\centering
\begin{subfigure}[b]{\textwidth}
    \centering
    \includegraphics[width=0.85\textwidth]{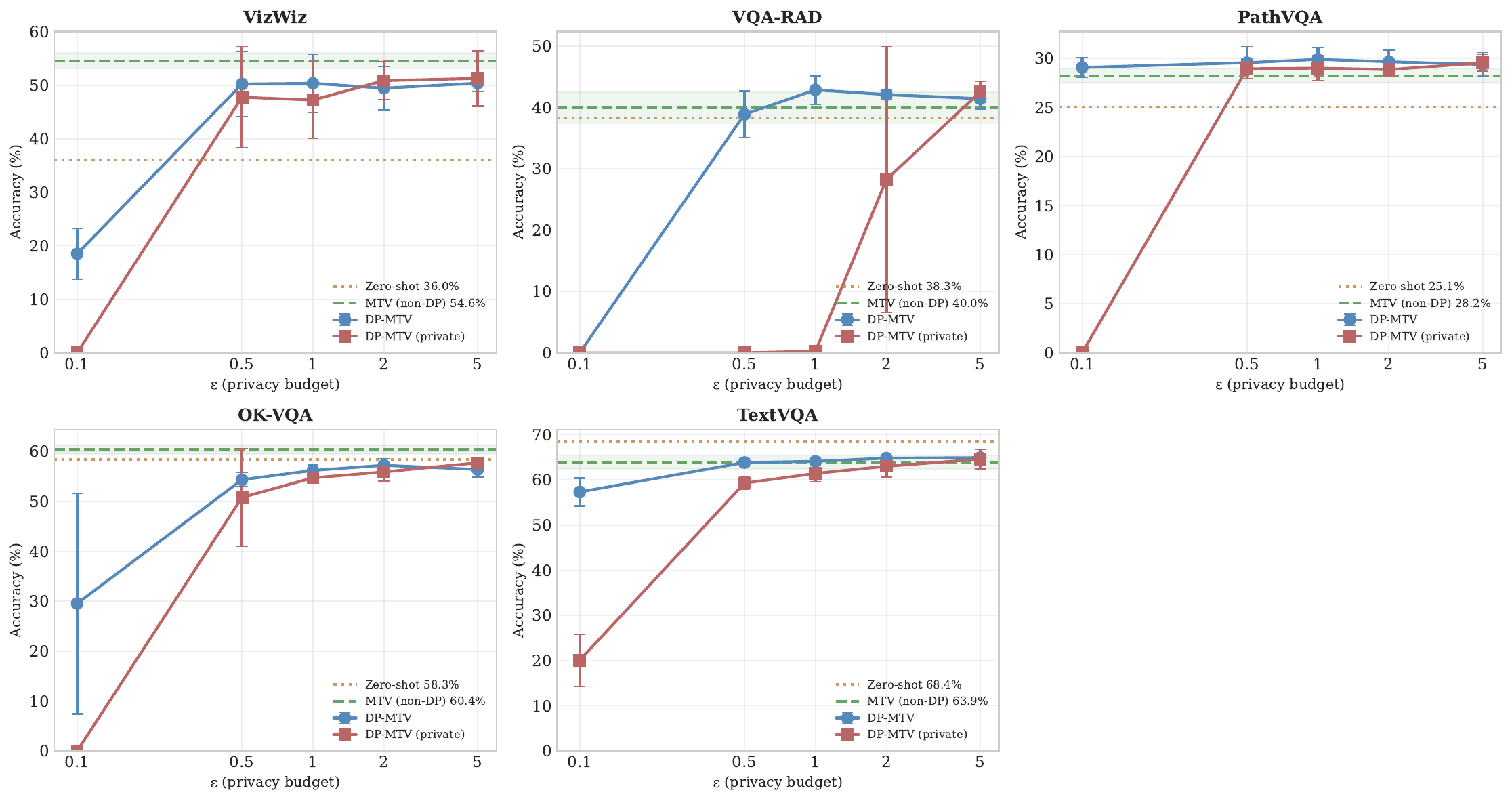}
    \caption{VQA Tasks}
    \label{fig:privacy_utility_vqa}
\end{subfigure}

\vspace{0.3cm}

\begin{subfigure}[b]{\textwidth}
    \centering
    \includegraphics[width=0.85\textwidth]{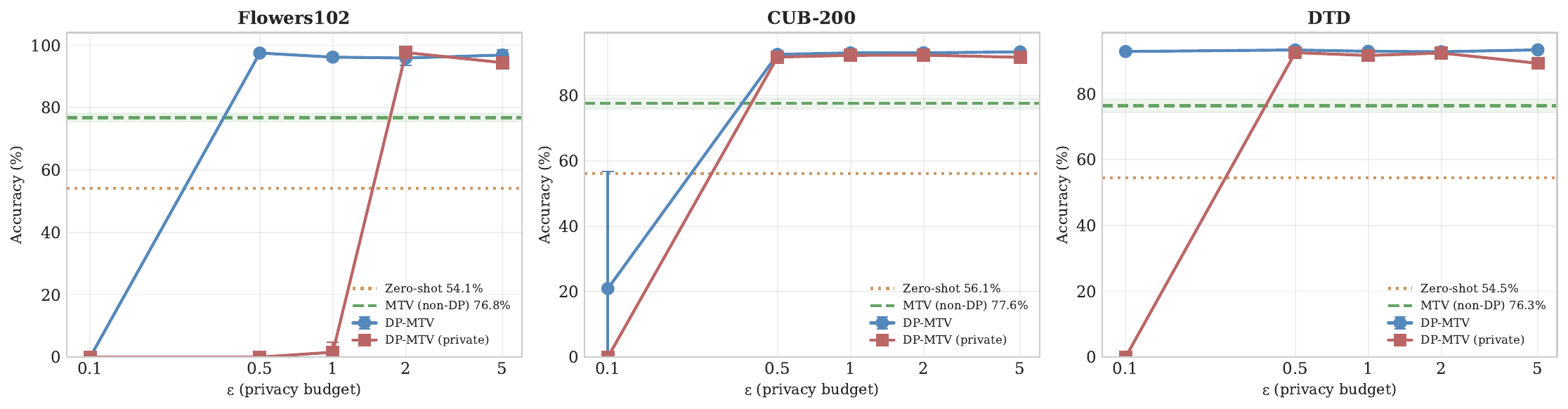}
    \caption{Classification Tasks}
    \label{fig:privacy_utility_classification}
\end{subfigure}
\caption{Privacy-utility tradeoffs for Qwen-VL across privacy budgets. 
Dashed lines: non-private MTV; dotted lines: zero-shot. 
(a)~On VQA, DP-MTV performs best on datasets where MTV provides the largest gains over zero-shot (e.g., VizWiz); performance varies by architecture (Figure~\ref{fig:model_comparison}). (b)~On classification, DP-MTV matches or exceeds MTV at practical privacy budgets.}
\label{fig:privacy_utility}
\end{figure*}

\paragraph{Datasets.} We evaluate on eight vision-language benchmarks (see Table~\ref{tab:datasets} in Appendix~\ref{app:datasets}). For VQA, we use VizWiz~\citep{gurari2018vizwiz} (visual questions from blind users), VQA-RAD~\citep{lau2018dataset} and PathVQA~\citep{he2020pathvqa} (medical VQA on radiology and pathology images), OK-VQA~\citep{marino2019visual} (knowledge-intensive questions), and TextVQA~\citep{singh2019towards} (questions requiring reading text in images). For classification, we use Flowers102~\citep{nilsback2008automated}, CUB-200~\citep{wah2011caltech}, and DTD~\citep{cimpoi2014describing} for fine-grained visual recognition. Following MTV, classification tasks are formulated as 2-way 1-shot problems: given one example each of classes A and B, the model predicts the query image's class.


\paragraph{Models.} We evaluate three vision-language models with different architectural foundations. \textbf{Qwen-VL}~\citep{bai2023qwen} is a LLaMA-based VLM pretrained on interleaved image-text data. \textbf{LLaMA3-ViLA-1.5-8B}~\citep{lin2024vila} (abbreviated as ViLA-1.5) uses LLaMA-3 as its language backbone. \textbf{Idefics2-8B}~\citep{laurenccon2024matters} is a Mistral-based model pretrained on web-scraped multimodal documents. All three models are designed for interleaved multimodal ICL, making them suitable testbeds for evaluating DP-MTV.

\paragraph{Baselines.} We compare against three baselines. \textbf{0-Shot} performs zero-shot inference without in-context examples. \textbf{Clean MTV} applies standard MTV without privacy constraints, establishing an upper bound. The gap between these baselines indicates the headroom available for DP-MTV. We evaluate two DP-MTV variants: \textbf{DP-MTV (Public)} uses examples from the validation set for head selection at zero privacy cost, while \textbf{DP-MTV (Private)} privatizes head selection using only the training data.

\paragraph{Evaluation.} We measure accuracy using the VQA evaluation metric~\citep{antol2015vqa} for VQA tasks and exact match for classification.All results are averaged over 3 random seeds. Privacy guarantees are reported as $(\varepsilon, \delta)$-differential privacy with $\delta=10^{-5}$.

\paragraph{Hyperparameters.} We extract activations from all 32 layers ($|\mathcal{S}|=32$) and partition each training set into $m=100$ disjoint chunks with clipping threshold $C=1.0$. For VQA datasets, each chunk contains one target plus $K=8$ demonstrations (9 examples total, using 900 examples per dataset). For classification, each chunk is a single 2-way example with $K=2$ support images (100 examples total). This represents 2-50\% of each training set depending on size. We evaluate privacy budgets $\varepsilon \in \{0.1, 0.5, 1, 2, 5\}$. For private head selection, we sample $\bar{k}=12$ candidate masks from the learned distribution, evaluate each on $B=100$ held-out validation examples, and use clipping threshold $C_{\text{sel}}=1.0$ for loss aggregation.

\subsection{ Results}

We evaluate DP-MTV on five VQA benchmarks and three classification datasets using three VLM architectures. Our experiments address three questions: \textit{ (1)~Does DP-MTV enable effective private in-context learning? (2)~When does DP-MTV work well? (3)~Does DP-MTV generalize across model architectures?} Figure~\ref{fig:privacy_utility} shows privacy-utility tradeoffs across the full range of privacy budgets.

\paragraph{Privacy-Utility Tradeoffs.}
Figure~\ref{fig:privacy_utility}a shows that DP-MTV enables effective private in-context learning on VQA tasks. On VizWiz, the public variant achieves 50.4\% at $\varepsilon=1.0$, retaining 92\% of MTV's performance (54.6\%). Performance improves further at higher privacy budgets, approaching MTV at $\varepsilon=5.0$. Other VQA datasets show varying degrees of success, which we analyze below. Figure~\ref{fig:privacy_utility}b shows classification results, where DP-MTV matches or exceeds MTV at moderate privacy budgets ($\varepsilon \geq 0.5$).

\paragraph{The Role of Baseline Gaps.}
For VQA tasks, a clear pattern emerges: DP-MTV performs best when MTV substantially improves over zero-shot. We call this difference the \emph{baseline gap}. When the gap is large, task vectors encode meaningful information that DP-MTV can preserve despite privacy noise. VizWiz exemplifies this with Qwen-VL's +19.2\% gap, providing headroom for DP-MTV to recover substantial gains. When gaps are small—VQA-RAD (+1.7\%), OK-VQA (+2.0\%) for Qwen-VL—there is limited signal to preserve, and DP-MTV shows modest differences from MTV. TextVQA presents an edge case where Qwen-VL's MTV underperforms zero-shot ($-4.5\%$ gap), indicating task vectors do not benefit this model-task combination; DP-MTV performs comparably to MTV, as expected.

Classification tasks exhibit different behavior. Despite large baseline gaps (+23--25\% for Qwen-VL), DP-MTV substantially exceeds MTV (e.g., Flowers102: 96.2\% vs 76.8\%). This reflects weak MTV baselines on 2-way classification tasks rather than exceptional DP-MTV performance, consistent with prior observations that
off-the-shelf VLMs underperform on classification~\citep{zhang2024visually}.

\paragraph{Model Architecture Effects.}
Figure~\ref{fig:model_comparison} reveals that baseline gaps and DP-MTV performance vary substantially across architectures. ViLA achieves larger gaps than Qwen-VL on knowledge-intensive tasks: +12.3\% vs +2.0\% on OK-VQA, and +4.3\% vs $-4.5\%$ on TextVQA, likely due to its LLaMA-3 foundation providing stronger world knowledge. Idefics2 shows mixed results: the public variant exceeds MTV on VizWiz (59.0\% vs 56.5\%) and 
classification, but fails on medical VQA (PathVQA, VQA-RAD), suggesting sensitivity to domain-specific tasks. Qwen-VL provides the most consistent results across datasets, while ViLA's private variant 
requires higher $\varepsilon$ for stable performance on some datasets. These differences suggest practitioners should consider both task requirements and architecture characteristics when deploying DP-MTV.

In some cases across both VQA and classification, DP-MTV matches or slightly exceeds non-private MTV. We hypothesize this may stem from known connections between differential privacy and robust statistics~\citep{dwork2009differential}: clipping can mitigate massive outliers in foundation model activations~\citep{diao2023outliers, chen2020understanding}, and Gaussian noise can act as randomized smoothing~\citep{bishop1995training, cohen2019certified}.


\paragraph{Deployment Recommendations.}
Based on our results, we offer two practical recommendations. First, the public-data variant is preferred when related auxiliary data exists: it achieves comparable or better accuracy while concentrating the full privacy budget on mean activations. Second, DP-MTV is most effective when MTV provides meaningful improvement over zero-shot—practitioners can estimate this gap on non-private data before committing to private deployment. The private-only variant enables fully private operation but typically requires $\varepsilon \geq 1.0$ for stable performance.

\subsection{Ablation Studies}
\paragraph{Clipping Threshold.}
Figure~\ref{fig:ablation_clip} shows accuracy as a function of clipping threshold $C$ on VizWiz with Qwen-VL at $\varepsilon=1.0$. Lower thresholds 
reduce sensitivity but increase bias from aggressive clipping; higher thresholds preserve more signal but require proportionally more noise. 
We find $C=1.0$ provides a favorable tradeoff, achieving 50.4\% accuracy. 
Performance degrades gradually at higher thresholds, confirming the method is not brittle to this hyperparameter choice.

\paragraph{Robustness to Hyperparameters.}
We vary the number of chunks $m \in \{50, 100, 500, 1000\}$ and demonstration shots $K \in \{2, 8, 16\}$. Performance remains stable across both ranges (Table~\ref{tab:ablation_robustness} in Appendix), indicating the method is not sensitive to these choices.

\begin{figure}[t]
\centering
\includegraphics[width=0.9\columnwidth]{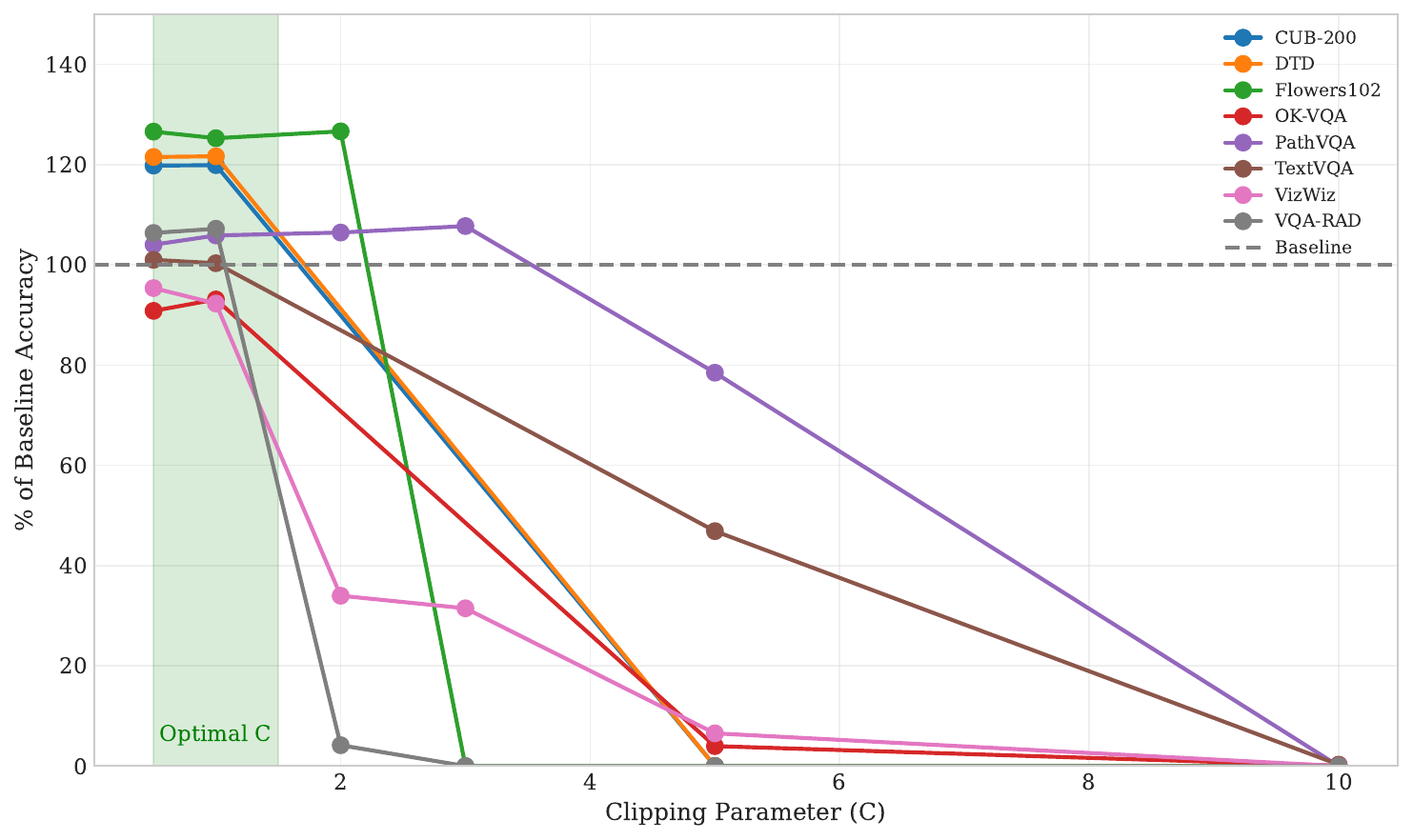}
\caption{Effect of clipping threshold $C$ on VizWiz accuracy (Qwen-VL, 
$\varepsilon=1.0$). $C=1.0$ balances signal preservation with noise 
calibration. Performance degrades gradually at higher thresholds.}
\label{fig:ablation_clip}
\end{figure}



\section{Related Work}



\subsection{Privacy Risks in Vision-Language Models}
As Vision-Language Models (VLMs) become ubiquitous, their susceptibility to privacy leakage has come under scrutiny. Adversarial attacks such as Membership Inference Attacks (MIA) can determine whether specific data points were used during training or in-context learning by analyzing model confidence scores \cite{wen2024membership, duan2024privacy}. In the multimodal domain, these risks are amplified; recent work has shown that VLMs can infer sensitive attributes (e.g., health status, location, demographics) from seemingly benign background details in images \cite{mendes2024granular, tomekce2024private, caldarella2024phantom}. Furthermore, prompt injection attacks via visual inputs can bypass textual safety guardrails, causing models to regurgitate private context data \cite{liu2024ocrbench, hou2025neighborhood}. These vulnerabilities necessitate rigorous privacy guarantees for any system processing sensitive user data.

\subsection{Differentially Private In-Context Learning}
To mitigate these risks, several frameworks have introduced Differential Privacy (DP) to In-Context Learning (ICL). Early approaches focused on text-only models, utilizing methods such as private synthetic demonstration generation \cite{panda2023dpicl, tang2023privacy} or ensemble-based aggregation like PATE \cite{wu2023privacypreserving}. More recent works have explored Local Differential Privacy (LDP) applied to labels \cite{zheng2024locally} or leveraging public auxiliary data to reduce privacy costs \cite{joo2025public}.

However, existing DP-ICL methods face a critical scalability barrier in the multimodal setting. Standard DP mechanisms account for privacy loss based on the information accessed; since a single image corresponds to hundreds of visual tokens \cite{bai2023qwenvl}, protecting multimodal contexts token-by-token leads to rapid privacy budget exhaustion. Our work addresses this by shifting the privacy mechanism from the token space to the activation space, enabling the aggregation of many-shot multimodal demonstrations without prohibitive privacy costs.

\subsection{Multimodal Task Vectors}
Multimodal Task Vectors (MTV) \cite{huang2024multimodal} represent a paradigm shift from token-based to activation-based ICL. By computing the mean activation of attention heads across hundreds of demonstrations, MTV compresses task-specific knowledge into a single steering vector that can be injected at inference time. This approach bypasses the context window limitations of standard ICL, enabling ``Many-Shot" learning. However, the original MTV framework lacks privacy guarantees; the task vectors are direct aggregates of raw private data, making them vulnerable to reconstruction or membership inference. \textsc{DP-MTV} bridges this gap by formalizing the aggregation process under $(\epsilon, \delta)$-DP.

\subsection{Our Contribution: Bridging the Gap}
Our work bridges the gap between scalable multimodal learning and rigorous privacy protection. While prior DP-ICL methods are constrained by context length and token costs, and standard MTV lacks privacy, \textsc{DP-MTV} introduces the first framework for \textit{Differentially Private Many-Shot Multimodal ICL}. By shifting the privacy mechanism from the token space to the activation space, we enable the aggregation of hundreds of demonstrations with a constant privacy cost. Furthermore, we provide empirical evidence that the privacy mechanisms themselves (clipping and noise) can serve as effective regularizers, enhancing performance on high-variance classification tasks.

\section{Conclusion}

We introduced DP-MTV, the first framework for differentially private  many-shot multimodal in-context learning. By operating in activation  space, DP-MTV aggregates patterns from hundreds of demonstrations and  privatizes the aggregate, incurring all privacy cost during construction and enabling unlimited inference queries thereafter.

Our experiments across eight benchmarks and three VLM architectures  show that formal $(\varepsilon, \delta)$-differential privacy is  achievable without sacrificing the benefit of learning from many  examples. At $\varepsilon=1.0$, DP-MTV preserves most of the gain from in-context learning compared to zero-shot baselines.

Performance varies across tasks, with stronger results on classification than open-ended VQA. Future work could explore tighter  composition for the fully-private setting, adaptive clipping strategies, or extensions to other activation editing methods.

\section*{Impact Statement}

This work enables vision-language models to learn from sensitive multimodal data with formal privacy guarantees. The primary benefit is allowing organizations in healthcare, finance, and legal domains to use many-shot in-context learning without exposing individual data to inference attacks.

Differential privacy provides guarantees for individuals but does not prevent inferences about groups or populations represented in the data. DP-MTV also inherits any biases present in the underlying VLM.

Our evaluation includes medical imaging datasets (VQA-RAD, PathVQA) and VizWiz, which contains images from blind and low-vision users. These represent real applications where privacy protection is essential.

\bibliography{ref}
\bibliographystyle{icml2026}

\newpage
\appendix
\onecolumn

\section{Experimental Details}
\label{app:exp_details}

\subsection{Dataset Statistics}
\label{app:datasets}

Table~\ref{tab:datasets} provides complete statistics for all datasets used in our experiments.

\begin{table}[h]
\centering
\caption{Dataset Statistics. All datasets use standard train/test splits from original papers.}
\label{tab:datasets}
\begin{tabular}{lcccc}
\toprule
Dataset & Domain & Train & Test & Task \\
\midrule
VizWiz & General VQA & 20,523 & 8,000 & Open-ended \\
VQA-RAD & Medical VQA & 1,793 & 451 & Open-ended \\
PathVQA & Medical VQA & 19,654 & 6,719 & Open-ended \\
OK-VQA & Knowledge VQA & 9,009 & 5,046 & Open-ended \\
TextVQA & Text VQA & 34,602 & 5,000 & Open-ended \\
\midrule
Flowers102 & Classification & 1,020 & 6,149 & 2-way 1-shot \\
CUB-200 & Classification & 5,994 & 5,794 & 2-way 1-shot \\
DTD & Classification & 3,760 & 1,880 & 2-way 1-shot \\
\bottomrule
\end{tabular}
\end{table}

With $m=100$ chunks and $K=8$ demonstrations per chunk (VQA) or $K=2$ (classification), we use approximately 500-900 training examples per dataset for private activation computation, corresponding to 2-20\% of each training set depending on dataset size.

\subsection{Implementation Details}
\label{app:implementation}

\paragraph{Computational Resources.} All experiments run on NVIDIA V100, A100 and H100 GPUs with 30GB and 80GB memory. Mean activation computation takes 1-2 hours per dataset depending on size. REINFORCE training for head selection takes 3-5 hours for the private variant and 2-4 hours for the public variant.


\paragraph{Public Data for Head Selection.} For the public-data variant, we use validation splits from related datasets when available. For VizWiz, we use VQA v2 validation data; for medical VQA tasks (VQA-RAD, PathVQA), we use samples from each other as public data since they share similar task structure despite different imaging modalities.

\section{Additional Algorithms}
This section provides the complete algorithm for private head selection used in the private-only variant (Section~\ref{sec:head_selection}).
\begin{algorithm}[t]
\caption{Private Head Selection via Noisy Top-k with Limited Domain}
\label{alg:private_head_selection}
\SetKwComment{Comment}{$\rhd$\ }{}

\KwIn{$\mathcal{D}_{\text{priv}}$ (private dataset), $F$ (VLM), $\bar{a}^{\text{priv}}$ (private mean activations), $\mathcal{S}$ (selected layers), $\varepsilon_{\text{sel}}, \delta_{\text{sel}}$ (privacy parameters), $C_{\text{sel}}$ (loss clipping threshold)}
\KwOut{$\mathbf{m} \in \{0,1\}^{|\mathcal{S}| \times H}$ (binary mask indicating which heads to modify)}

\Comment*[l]{Phase 1: Learn Bernoulli distributions via REINFORCE}
$\theta \gets \text{REINFORCE}(\mathcal{D}_{\text{priv}}, F, \bar{a}^{\text{priv}})$ \Comment*[r]{Train policy on private data}

\vspace{0.6em}

\Comment*[l]{Phase 2: Sample candidate masks (limited domain)}
$\mathcal{M} \gets \emptyset$ \Comment*[r]{Store sampled masks}
\For{$j \gets 1$ \KwTo $2000$}{
    $\lambda_j \sim \text{Bernoulli}(\sigma(\theta))$ \Comment*[r]{Sample binary mask}
    $\mathcal{M} \gets \mathcal{M} \cup \{\lambda_j\}$
}
$\{\mathbf{m}_1, \ldots, \mathbf{m}_{\bar{k}}\} \gets \text{Top-}\bar{k}\text{-Frequent}(\mathcal{M})$ \Comment*[r]{Retain $\bar{k}=12$ most frequent}

\vspace{0.6em}

\Comment*[l]{Phase 3: Compute aggregate scores with clipping}
$\mathcal{D}_{\text{val}} \gets$ Sample $B=100$ examples from $\mathcal{D}_{\text{priv}}$ \Comment*[r]{Validation set}
\For{$i \gets 1$ \KwTo $\bar{k}$}{
    $s_i \gets 0$\;
    \For{$(x, q, a) \in \mathcal{D}_{\text{val}}$}{
        Replace activations in $F$ at heads indicated by $\mathbf{m}_i$ with $\bar{a}^{\text{priv}}$\;
        $\ell \gets \text{CrossEntropy}(F(x, q), a)$ \Comment*[r]{Per-example loss}
        $s_i \gets s_i + \min(\ell, C_{\text{sel}})$ \Comment*[r]{Clip and accumulate}
    }
}

\vspace{0.6em}

\Comment*[l]{Phase 4: Apply Gumbel mechanism for private selection}
\For{$i \gets 1$ \KwTo $\bar{k}$}{
    $g_i \sim \text{Gumbel}(0, C_{\text{sel}}/\varepsilon_{\text{sel}})$ \Comment*[r]{Sample Gumbel noise}
    $\tilde{s}_i \gets s_i + g_i$ \Comment*[r]{Add noise to score}
}

$i^* \gets \arg\min_{i \in [\bar{k}]} \tilde{s}_i$ \Comment*[r]{Select mask with lowest noisy loss}

\Return{$\mathbf{m}_{i^*}$}
\end{algorithm}

\section{Additional Experimental Results}
\label{app:additional_results}

\subsection{Robustness to Hyperparameters}
\label{app:ablations}

\begin{table}[h]
\centering
\caption{Sensitivity to hyperparameters on VizWiz (Qwen-VL, $\varepsilon=1.0$, $C=1.0$). Performance is stable across both number of shots $K$ and number of chunks $m$.}
\label{tab:ablation_robustness}
\small
\begin{tabular}{lccccc}
\toprule
& \multicolumn{4}{c}{Number of Chunks $m$} \\
\cmidrule(lr){2-5}
Method & 50 & 100 & 500 & 1000 \\
\midrule
DP-MTV & 49.8±8.6 & 49.1±10.3 & 50.6±3.9 & 48.7±5.9 \\
DP-MTV (priv) & 42.3±14.0 & 77.2±5.7 & 46.7±2.3 & 46.8±3.5 \\
\bottomrule
\end{tabular}

\vspace{0.5cm}

\begin{tabular}{lcccc}
\toprule
& \multicolumn{3}{c}{Number of Shots $K$} \\
\cmidrule(lr){2-4}
Method & 2 & 8 & 16 \\
\midrule
DP-MTV & 51.1±2.6 & 52.0±1.2 & 50.6±3.3 \\
DP-MTV (priv) & 48.0±1.6 & 50.4±0.6 & 50.5±1.9 \\
\bottomrule
\end{tabular}
\end{table}



\section{Privacy Analysis Proofs}
\label{app:proofs}

\begin{proof}[Proof of Theorem~\ref{thm:task_vector_privacy}]
We analyze privacy under the bounded contribution model where neighboring datasets differ in one example's data within a single chunk.

\textit{Sensitivity bound.} Each example appears in exactly one chunk. Per-layer clipping ensures $\|\tilde{a}_i^{(l)}\|_2 \leq C$ for each layer $l \in \mathcal{S}$. The full clipped tensor thus satisfies:
\[
\|\tilde{a}_i\|_2 = \sqrt{\sum_{l \in \mathcal{S}} \|\tilde{a}_i^{(l)}\|_2^2} \leq \sqrt{|\mathcal{S}| \cdot C^2} = \sqrt{|\mathcal{S}|} \cdot C
\]
Changing one example affects one chunk's contribution to the mean $\bar{a} = \frac{1}{m}\sum_i \tilde{a}_i$, yielding $\ell_2$-sensitivity $\Delta_2 = \sqrt{|\mathcal{S}|} \cdot C / m$.

\textit{Privacy guarantee.} The Gaussian mechanism with $\sigma = \Delta_2 \sqrt{2\ln(1.25/\delta)} / \varepsilon_{\text{tv}}$ satisfies $(\varepsilon_{\text{tv}}, \delta)$-DP by standard results~\citep{dwork2014algorithmic}.
\end{proof}

\begin{proof}[Proof of Theorem~\ref{thm:public_variant}]
Head selection operates entirely on $\mathcal{D}_{\text{pub}}$, which is independent of $\mathcal{D}_{\text{priv}}$, contributing $(0,0)$ to privacy cost. The released artifacts $(\bar{a}^{\text{priv}}, \mathbf{m})$ thus inherit the $(\varepsilon_{\text{tv}}, \delta)$ guarantee from Theorem~\ref{thm:task_vector_privacy}. Inference applies deterministic functions to these artifacts; by post-processing, this preserves $(\varepsilon_{\text{tv}}, \delta)$-DP for unlimited queries.
\end{proof}

\begin{proof}[Proof of Theorem~\ref{thm:private_variant}]
Mean activations satisfy $(\varepsilon_{\text{tv}}, \delta)$-DP by Theorem~\ref{thm:task_vector_privacy}.

For head selection, we first use REINFORCE to learn a Bernoulli distribution over heads, then sample $\bar{k}$ candidate masks from this distribution. For each candidate, we compute an aggregate score by evaluating it on private validation data, with per-example losses clipped to $[0, C_{\text{sel}}]$. This clipping bounds the sensitivity of each score to $C_{\text{sel}}$. We then apply the Gumbel mechanism, adding noise with scale $C_{\text{sel}}/\varepsilon_{\text{sel}}$ to each score and selecting the mask with the lowest noisy loss.

The Gumbel mechanism (equivalently, Report Noisy Max) with bounded-sensitivity scores provides $(\varepsilon_{\text{sel}}, 0)$-DP for the released mask~\citep{durfee2019practical}. By basic composition, releasing both $\bar{a}^{\text{priv}}$ and $\mathbf{m}$ satisfies $(\varepsilon_{\text{tv}} + \varepsilon_{\text{sel}}, \delta + 0) = (\varepsilon_{\text{tv}} + \varepsilon_{\text{sel}}, \delta)$-DP. Inference is post-processing.
\end{proof}

\end{document}